\documentclass[letterpaper]{article} 
\usepackage{aaai25}  
\usepackage{times}  
\usepackage{helvet}  
\usepackage{courier}  
\usepackage[hyphens]{url}  
\usepackage{graphicx} 
\usepackage{natbib}  
\usepackage{caption} 
\usepackage{algorithm}
\usepackage{algorithmic}

\usepackage[utf8]{inputenc} 
\usepackage{booktabs}       
\usepackage{amsfonts}       
\usepackage{nicefrac}       
\usepackage{microtype}      
\usepackage{xcolor} 

\usepackage{amsthm}         
\usepackage{zzhang}

\usepackage[capitalize]{cleveref}

\usepackage{graphicx}
\usepackage{booktabs}
\usepackage{adjustbox}
\usepackage{listings}
\definecolor{mygreen}{RGB}{147, 178, 106}
\definecolor{mypurple}{RGB}{141, 116, 164}

\newtheorem{theorem}{Proposition}
\newtheorem{proposition}[theorem]{Proposition}

\newcommand{\best}[1]{\textbf{\textcolor{red}{#1}}}
\newcommand{\secondbest}[1]{\textcolor{blue}{#1}}

\makeatletter
\newtheorem*{rep@theorem}{\rep@title}
\newcommand{\newreptheorem}[2]{%
  \newenvironment{rep#1}[1]{%
    \def\rep@title{#2 \ref{##1}}%
    \begin{rep@theorem}}%
    {\end{rep@theorem}}}
\makeatother
\newreptheorem{proposition}{Proposition}

\DeclareCaptionStyle{ruled}{labelfont=normalfont,labelsep=colon,strut=off} 
\floatstyle{ruled}
\newfloat{listing}{tb}{lst}{}
\floatname{listing}{Listing}
\title{Rethinking State Disentanglement in Causal Reinforcement Learning}
\author{
    Haiyao Cao\textsuperscript{\rm 1}\equalcontrib, 
    Zhen Zhang\textsuperscript{\rm 1}\equalcontrib, 
    Panpan Cai\textsuperscript{\rm 2}, 
    Yuhang Liu\textsuperscript{\rm 1}, 
    Jinan Zou\textsuperscript{\rm 1}, 
    Ehsan Abbasnejad\textsuperscript{\rm 1},\\
    Biwei Huang\textsuperscript{\rm 3}, 
    Mingming Gong\textsuperscript{\rm 4}, 
    Anton van den Hengel\textsuperscript{\rm 1}, 
    Javen Qinfeng Shi\textsuperscript{\rm 1}\thanks{Corresponding author}
}
\affiliations{
    \textsuperscript{\rm 1}Australia Institute for Machine Learning, The University of Adelaide\\
    \textsuperscript{\rm 2}Shanghai Jiao Tong University\\
    \textsuperscript{\rm 3}University of California San Diego\\
    \textsuperscript{\rm 4}The University of Melbourne\\
    \{haiyao.cao,zhen.zhang02,yuhang.liu01,jinan.zou,ehsan.abbasnejad,anton.vandenhengel,javen.shi\}@adelaide.edu.au,
    cai\_panpan@sjtu.edu.cn,
    bih007@ucsd.edu,
    mingming.gong@unimelb.edu.au
}

\usepackage{bibentry}

\begin{document}

\maketitle

\begin{abstract}

One of the significant challenges in reinforcement learning (RL) when dealing with noise is estimating latent states from observations. Causality provides rigorous theoretical support for ensuring that the underlying states can be uniquely recovered through identifiability. Consequently, some existing work focuses on establishing identifiability from a causal perspective to aid in the design of algorithms. However, these results are often derived from a purely causal viewpoint, which may overlook the specific RL context. We revisit this research line and find that incorporating RL-specific context can reduce unnecessary assumptions in previous identifiability analyses for latent states. More importantly, removing these assumptions allows algorithm design to go beyond the earlier boundaries constrained by them. Leveraging these insights, we propose a novel approach for general partially observable Markov Decision Processes (POMDPs) by replacing the complicated structural constraints in previous methods with two simple constraints for transition and reward preservation. With the two constraints, the proposed algorithm is guaranteed to disentangle state and noise that is faithful to the underlying dynamics. Empirical evidence from extensive benchmark control tasks demonstrates the superiority of our approach over existing counterparts in effectively disentangling state belief from noise. The code is avaliable at \url{https://github.com/Haiyao-Nero/causal-rl-rethinking}
\end{abstract}

\section{Introduction}
Model-Based Reinforcement Learning (MBRL) is a specialized form of RL that combines a learned environment model, often called a ‘world model,’ with a planning algorithm to guide the agent’s decision-making process \citep{polydoros2017survey}. The world model forecasts future states and rewards by simulating the agent’s interactions with the environment. The planning algorithm leverages these predictions to explore possible scenarios and optimize the agent’s policy.
Recently, there has been a surge of interest in learning latent-space world models \citep{ha2018world,hafner2019dreamer,hafner2020dreamerv2,sekar2020planning,zhu2022invariant,choi2023local,he2023frustratingly,hafner2023dreamerv3}. These models map high-dimensional observation data, such as images, to an abstract latent representation, effectively capturing the dynamics of the environment within this learned latent space. These approaches offer the advantage of simplifying complex environments, reducing computational demands, and potentially enhancing policy generalization.
A common assumption in these methods is that observed data is noise-free. However, in practical applications where noise is prevalent, the effectiveness of these techniques is significantly hindered. This challenge arises from the difficulty of disentangling the reward-related aspects or ``signals'', which cannot be easily separated from the noise-related components. Such entanglement can introduce distractions during the learning process, ultimately leading to suboptimal performance \citep{efroni2021provable}.

Disentangling latent states from noisy observations has attracted significant attention across various domains in machine learning, including domain adaptation and generalization. A key challenge lies in ensuring that the underlying latent states can be uniquely recovered. This issue is closely tied to the notion of identifiability in causality, such as in nonlinear Independent Component Analysis (ICA) \citep{khemakhem2020variational} and causal representation learning \citep{chen2018isolating,scholkopf2021toward,liu2022identify,liu2024identifiable,zhang2024identifiability}. Leveraging this connection, some existing works \citep{huang2022action, liu2023learning} have sought to provide identifiability guarantees from a causal perspective for estimating latent states in RL. Most importantly, such guarantees often provide valuable insights for designing improved algorithms in RL. As a result, these algorithms often consistently demonstrate significant performance gains across various benchmarks. 

Despite the above, previous works \cite{huang2022action, liu2023learning} adopt a purely theoretical viewpoint, often overlooking the RL-specific context. For instance, 1) previous works often assume that the underlying states can be divided into certain subsets, with each subset being independent of the others. Under this assumption, previous works can achieve identifiability guarantees one by one to identify the independence subset. However, in RL, such detailed identifiability is unnecessary, as the optimal policy only requires the latent states to be recovered up to an 
invertible mapping, without splitting them into independent subsets. More importantly, these independence assumptions on generative models may be unrealistic and often get violated in real-world RL applications. In addition, 2) For purely causal representation, they often assume that there is an invertible observation function, which can not be satisfied in general POMDP. In POMDPs, the observation function maps the latent states to observations, but this mapping is often many-to-one, meaning multiple latent states can produce the same observation. As a result, the observation function is not invertible because it is impossible to uniquely recover the latent state from an observation. This inherent ambiguity in POMDPs makes the assumption of an invertible observation function unrealistic in practice.

Adopting an RL-specific perspective, we relax the two assumptions mentioned earlier to obtain more general identifiability results. Specifically, we demonstrate that leveraging rewards and transitions sufficiently eliminates the need for the assumption that latent states can be split into independent subsets for recovery. Additionally, after converting a POMDP with a non-invertible observation function into an equivalent belief-MDP with an invertible observation function, we establish that identifiability can be still achieved in the context of belief-MDP, thereby relaxing the assumption of an invertible observation function in general POMDPs. More importantly, these theoretical analyses not only relax the assumptions used in previous works but also provide two key insights for algorithm design: 
1) The transition preservation and reward preservation constraints are sufficient to ensure the disentanglement of state and noise, without the assumption that the states can be split into subsets. 
2) In the context of belief-MDPs, the noise belief to be conditionally independent of the future state belief given the current state belief is sufficient for disentangling state and noise, without requiring the noise to be independent of the state. Finally, we integrate our theoretical analysis into our algorithm design and propose a new method. In theory, with strong theoretical backing, the proposed method is consistent with the underlying dynamics of both transition and reward, thereby ensuring effective disentanglement of state and noise. Technically, by incorporating the two insights mentioned above, the proposed method relaxes previous assumptions and overcomes the limitations of existing approaches constrained by those assumptions. Experimental evaluations on the DeepMind Control Suite and Robodesk demonstrate that the proposed method consistently outperforms previous methods by a significant margin.

\section{Preliminaries and Related Works}

\subsection{Partial Observable Markov Decision Process}

In this paper, we consider the theory of causal identifiability on the Partially Observable Markov Decision Process (POMDP).
POMDP can be defined as a 7-tuple $(\mathcal{S},\mathcal{A},\mathcal{O},\mathcal{T},\mathcal{M},R, \gamma)$ consists of state, action and observation spaces, $\mathcal{S},\mathcal{A}$ and $\mathcal{O}$; transition function $\mathcal{T}=p(s'|a, s)$; observation function $o=\mathcal{M}(s, z)$ where $z$ denotes the noise; reward function $R(s',a,s)$; and discount factor  $\gamma\in (0, 1]$. The goal of POMDP is to maximize the expected accumulated reward, a.k.a. return. For a POMDP, the underlying MDP denoted by a 5-tuple $(\mathcal{S}, \mathcal{A}, \mathcal{T}, R, \gamma)$ plays an important role in theoretical analysis. 

One key problem of POMDP is to estimate the underlying states using observations. Particularly for noisy and redundant observations such as in \cite{fu2021learning,wang2022denoised,zhang2020invariant}, it requires disentanglement of state and noise to obtain a compact representation. 

\subsection{Representation learning in RL}
Various algorithms including bisimulations, contrastive augmentation and reconstruction have recently advanced markovian representation learning in MBRL. Bisimulations \citep{gelada2019deepmdp,zhang2020invariant} and contrastive augmentation methods \citep{he2023frustratingly,laskin2020curl,misra2020kinematic,deng2022dreamerpro} focus on state differentiation. Reconstruction  \citep{watter2015embed,wahlstrom2015pixels} is pivotal in compressing observations into effective representations. 
Studies. From rich observations with constraints, \citet{dann2018oracle,du2019provably,huang2022action,eysenbach2021robust} have shown that minimal POMDP representations can be derived. The effectiveness of reconstruction-based MBRL is exemplified by Dreamer \citep{hafner2019dreamer}, DreamerV2 \citep{hafner2020dreamerv2}, and DreamerV3 \citep{hafner2023dreamerv3}.
Further advancements have been made in denoising world models such as TIA \citep{fu2021learning}, Denoised-MDP \citep{wang2022denoised}, and IFactor \citep{liu2023learning}, which are built upon Dreamer models.

\section{Motivation: Mind the Gap in Causal RL}

A key challenge in RL is disentangling high-level latent states and latent noise in low-level observations. This issue closely aligns with an important area in causality: causal representation learning, which seeks to uncover latent causal variables from observable data \cite{liu2022identify,liu2024identifiable,zhang2024identifiability}. This natural connection has led to the development of numerous methods that apply the principles of causal representation learning to analyze RL settings \cite{huang2022action,liu2023learning}. For instance, by using identifiability analysis in causal representation learning, one can offer new insights that guide the design of RL methods. Consequently, these methods, benefiting from the guarantees of identifiability, often significantly outperform earlier approaches that rely on heuristic constraints, such as employing the VAE framework.

However, we question the general application of causal representation learning to the RL setting, primarily due to the following concern: the detailed identifiability analysis often pursued in causal representation learning may be overly precise, and such rigor is generally not required in RL contexts. For example, the identifiability results in \citet{huang2022action,liu2023learning} require that latent states should be split into some subset, and each subset can be identified up to a nonlinear and invertible mapping. However, in RL, such detailed identifiability is unnecessary, as the optimal policy only requires the latent states to be recovered up to an invertible mapping, without splitting them into specific subsets. These additional requirements in identifiability analysis within causal representation learning often stem from strong assumptions on generative models, which may not hold in real-world applications.

Beyond exploring RL from a general perspective of causal representation learning, we propose adopting an RL-specific causal representation viewpoint. Specifically, our goal is to analyze the identifiability of latent states without the need to divide them into specific subsets, as required in prior works \citet{huang2022action,liu2023learning}. The rationale behind the goal is that: 1) treating latent states as a whole is sufficient for contributing to the optimal policy, without splitting them into specific subsets, and 2) splitting them into independent subsets, an assumption that may not hold in real-world applications, is unnecessary. We have the following Proposition.

\begin{proposition}\label{propos:disentangle}
 Given POMDP $(\mathcal{S},\mathcal{A},\mathcal{O},\mathcal{T},\mathcal{M},\Rcal, \gamma)$, if the observation function $\Mcal$ is invertible. 
    Let $g:\Ocal\mapsto \Scal\times\Zcal$ be an invertible state estimation function. If the following conditions hold:
    \begin{enumerate}
    \item \textbf{No redundancy: } For different states $s_1$ and $s_2 \in \Scal$ , there exists an policy $\pi$ that $V^\pi(s_1)\neq V^\pi(s_2)$. 
    \item \textbf{Transition preservation: } For any observation $o$ and any action $a$, let $\hat{s}, \hat{z} = g(o)$, we have
    \begin{equation}\label{eq:noise_state_independ}
        p(o'|a,o)=p(\hat{s}'|a,\hat{s})p(\hat{z}'|\hat{z}),
    \end{equation}
    \item \textbf{Reward preservation:} For any state observation pair $(s, o)$ with $\hat{s}, \hat{z} = g(o)$, for any action $a$, we have
    \begin{equation}\label{eq:reward_reservation}
        \exists\hat{\Rcal} \text{ s.t. } \Rcal(o, a, o')=\Rcal(s, a, s')=\hat{\Rcal}(\hat{s}, a, s')
    \end{equation}
  \end{enumerate}
  then the estimated MDP $(\hat{\Scal}=\{\hat{s}|\exists o, \text{s.t. } \hat{s}, \hat{z} = g(o) \}, \Acal, \hat{\Tcal} = p(\hat{s}'|a, \hat{s}), \hat{\Rcal}, \gamma) $ is equivalent to the underlying MDP $(\mathcal{S},\mathcal{A},\mathcal{T},\Rcal, \gamma)$, and the state estimation function $g$ disentangles state and noise. 
\end{proposition}
\begin{proof}
    The proof sketch is as follows. For invertible $\Mcal$, we can use the observation as the state to construct an MDP, which is value equivalent to the underlying MDP. Using a similar technique, we can prove that the constructed MDP which uses observation as the state is value equivalent to the estimated MDP. Then the estimated MDP is value equivalent to the underlying MDP. Consider the mapping $f:\hat{\Scal}\mapsto\Scal$, if it maps $\hat{s}$ to multiple $s\in\Scal$, $\hat{s}$ must exhibit more than one value under certain policies, contradicting the fundamental principle of MDP, where, for a given policy, the value of a state is deterministic. Similar, by the No redundancy condition $f^{-1}$ can not map $s\in\Scal$ to multiple $\hat{\Scal}$.
\end{proof}

\paragraph{Remarks}
RL aims at optimizing future returns. Therefore if two different instances of a state, derived from the same state representation, yield identical returns under any policy, these instances are indistinguishable, indicating redundancy within the representation.  Then it is natural to assume that the underlying true state representation is the most compact one and free of redundancy. Furthermore,  if an alternative state representation yields transitions and rewards consistent with the underlying one, it can be considered equivalent to the underlying state representation.

Intuitively, if under any policy the values of two states are equal, this suggests redundancy in the current state representation, implying the possibility of achieving a more compact representation. Thus it would be natural to assume that there is no redundancy in the underlying state representation. 
These two conditions are used to enforce that any intervention $\hat z$ would not affect the accumulated rewards.

\paragraph{Insights}  \Cref{propos:disentangle} implies that we can recover latent states as a whole (e.g., up to an invertible mapping) without splitting them into independent subsets and recovering each subset separately, as done in previous works. The reason we can achieve these more general results is that: from the RL perspective, rewards and transitions sufficiently contribute to the recovery of latent states. In contrast, previous work, when viewed purely from a causality perspective, neglects this point and instead relies on the additional assumption that latent states can be split into independent subsets for recovery.

\paragraph{Discussion} Although we have adopted a RL-specific causal representation viewpoint in \Cref{propos:disentangle}
 to achieve more general identifiability results compared to the purely causal perspectives in \citet{huang2022action,liu2023learning}, there remains an assumption that could stem from a purely causal viewpoint—namely, the invertibility of the observation function $\Mcal$, which can not be satisfied in general POMDP. From an RL-specific causal perspective, it is possible to convert a POMDP with a non-invertible observation function to an equivalent belief-MDP where the observation function becomes invertible (details shown in next section), however, in the belief-MDP the belief of latent noise and latent states can not be independent. To address this, we extend \Cref{propos:disentangle} to accommodate general POMDP. From the RL-specific causal perspective, the latent states can be viewed as a group of variables that can fully determine the reward and changes independently, thus the transition preservation condition in \Cref{propos:disentangle} can be relaxed as follows.
\begin{proposition}\label{propos:conditional_independent}
    Under the assumptions of Proposition \ref{propos:disentangle}, the transition preservation condition in Proposition \ref{propos:disentangle} can be extended as $\hat{z}$ is conditional independent with future state $\hat{s}'$ given $\hat{s}$, and the results of Proposition \ref{propos:disentangle} still hold.
\end{proposition}
\paragraph{Summary} 
\Cref{propos:disentangle} suggests that we can use the \Cref{eq:noise_state_independ} and \Cref{eq:reward_reservation} as constraints to ensure the transition preservation and reward preservation \footnote{Trivially by restricting the size of estimated state space the ``no redunancy'' constraints can be enforced. }, 
thereby enabling the learned state estimation function being able to disentangle latent states and latent noise from observations. \Cref{propos:conditional_independent} suggests that a more general class of transition preservation constraints, e.g., refer to details in Eq. \eqref{eq:state_preservation_new}, can be provided to disentangle latent states and latent noise for general POMDP where the beliefs of latent states and latent noise are dependent.


\section{Disentangling State in Belief World Model}
\begin{figure*}[t]
    \centering
    \includegraphics[width=1\textwidth]{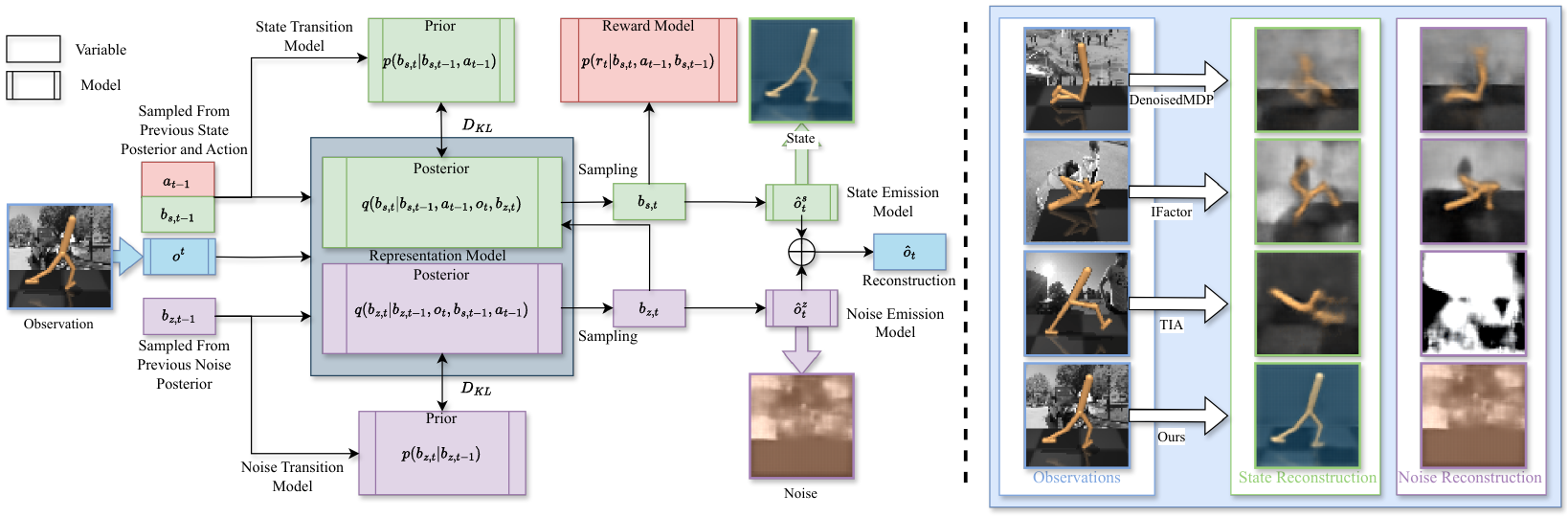}
    \caption{{The overall illustration of our world model consists of \textcolor{mygreen}{state belief RSSM} (in \textcolor{mygreen}{green}) and \textcolor{mypurple}{noise belief RSSM} (in \textcolor{mypurple}{purple}). They are differentiated by the action, reward model and emission models, which are the key parts of maintaining reward preservation and transition preservation. Furthermore, RSSM for the belief of noise does more than merely model the transitions of noise, it also encapsulates the inherent uncertainty within the system, due to Bayesian nature. Visualizations of the recovered state and noise clearly show that our world model significantly outperforms {TIA} \cite{fu2021learning}, {DenoisedMDP} \cite{wang2022denoised} and {IFactor} \cite{liu2023learning} in disentangling state and noise.}}
    \label{fig:enter-label}
\end{figure*}

For general POMDP, it is well-known that exactly recovering the latent state is not possible. In this case, we can convert the POMDP to an equivalent belief-MDP \citep{rodriguez1999reinforcement}, and recover the state belief, which is defined as the conditional distribution of state given all previous observations and actions as follows
\begin{align}\label{eq:belief_def}
 b_t(s_t, z_t) &  =  p(s_t, z_t| o_{\leq t}, a_{< t}) \\ 
 \propto &  p(o_t| s_{t}, z_{t}, a_{t-1}) \sum_{s_{t-1}, z_{t-1}}b_{t-1}(s_{t-1},z_{t-1}) \notag \\ 
 & \times p(s_t, z_t|a_{t-1}, s_{t-1}, z_{t-1})\notag.
\end{align}
It is important to note that in the belief representation, even when the dynamics of state and noise are fully independent, the joint belief $b_t(s_t,z_t)$ cannot be factorized into a product of state belief and noise belief. This is due to the fact that states $s_t$ and noise $z_t$ are not conditionally independent given the observation $o_t$. In this case, the dynamics of beliefs on state and noise are not independent. 
However, we can factorize the beliefs of state and noise as
\begin{subequations}\label{eq:state_and_noise_belief}
\begin{align}
    b_t(s_t) & = p(s_t | o_{\leq t}, a_{< t}) \\ 
    & \propto p(o_{t}|s_{t}, a_{t-1}) \sum_{s_{t-1}}p(s_t|a_{t-1}, s_{t-1})b_{t-1}(s_{t-1}), \notag \\
    b_t(z_t|s_t)& = b_t(s_t, z_t) / b_t(s_t).
\end{align}    
\end{subequations}

By the above derivation, the state belief can be determined by previous belief and current raw observation. Meanwhile the reward in belief space 
\begin{align}\label{eq:reward}
    &\bar{R}(b_t, a_t, b_{t+1}) \notag\\
    & = \sum_{\mathbf{s}, \mathbf{z}}R(s_t, a_t, s_{t+1})b_t(s_t, z_t)b_{t+1}(s_{t+1}, z_{t+1}) \notag \\
    & = \sum_{s_t}R(s_t, a_t, s_{t+1})b_t(s_t)b_{t+1}(s_{t+1})  \\ 
    & = \sum_{s_t}R(s_t, a_t, s_{t+1})p(s_t|o_{\leq t}, a_{< t})p(s_{t+1}|o_{\leq t+1}, a_{< t+1})\notag
\end{align}
can also be fully determined using the state belief. Thus by \Cref{propos:conditional_independent}, as $b_t(z_t|s_t)$ is conditional independent of future state belief  $b_{t+1}(s_{t+1})$ given current state belief $b_t(s_t)$, 
a belief estimator that satisfies the following transition preservation condition 
\begin{align}\label{eq:state_preservation_new}
    p(b_{s,t}|a_{t-1}, b_{s_{t-1}})p(b_{z,t}|b_{z,t-1},b_{s,t}) = p(o_t|a_t, o_{\leq t-1}, a_{\leq t}),
\end{align}
and the reward preservation condition \eqref{eq:reward} would be sufficient to disentangle state belief and noise belief from noisy history $[o_{\leq t}, a_{<t-1}]$, where  $b_{s,t}$ denotes the parameters of $b_t(s_t)$ and $b_{z,t}$ denotes the parameters of $b_t(z_t|s_t)$ in the constraint \eqref{eq:state_preservation_new}, 

\subsection{World Model and Learning}
Based on our theory, we design a Recurrent State Space Model (RSSM) to estimate state belief in noisy history $[O_{\leq t}, a<{t}]$ that changes independently and can solely determine the expected reward. In our implementation, the beliefs are parametrized as Gassian distributions using Variational Auto Encoder (VAE) and thus we use $b_{s,t}$ to denote the parameters of state belief and use $b_{z,t}$ to denote the parameters of noise belief.

In our world model, naturally by \Cref{eq:belief_def}, \eqref{eq:state_and_noise_belief} and \eqref{eq:reward}, the dynamics can be defined as
\begin{subequations}\label{eq:world_model}
\begin{align}
&\textbf{State Transition Model:}&   p_\psi(b_{s,t}|b_{s,t-1},{a_{t-1}}),\\
&\textbf{Reward Model:} &  p_{\theta}(r_t|b_{s,t},a_{t-1},b_{s,t-1}),\\
&\textbf{Noise Transition Model:}&   p_\psi(b_{z,t}|b_{z,t-1}, b_{s,t}),\\
&\textbf{Emission Model:} & p_{\theta}(o_t|b_{s,t},b_{z,t}).
\end{align}
\end{subequations}
Meanwhile, particularly for RL problems with the noisy image as observation, as shown in \Cref{fig:enter-label} and \Cref{samples}, can often be roughly modelled as additive noise. In this case, the emission model can be factorized into two asymmetric reconstructions, denoted as
\begin{align}
     p_{\theta}(o_t|b_{s,t},b_{z,t}) = p(o_t|o_t^s, o_t^z)p(o_t^z|b_{z,t})p(o_t^s|b_{z,t}).
\end{align}
Put these together, we can employ Maximum Likelihood Estimation (MLE) to align the model with data distribution using the following loss
\begin{align} \label{eq:MLE_loss}
    \mathcal{L}_{\text{MLE}} = -\mathbb{E}_{o_{\le t}, a_{<t}, r_t\sim p_{data}}\log p_{\theta,\psi}(o_t, r_t|o_{\le t - 1}, a_{< t - 1}, t).
\end{align}
As the true underlying dynamics satisfy \Cref{eq:belief_def}, \eqref{eq:state_and_noise_belief} and \eqref{eq:reward}, the optimal MLE estimation satisfies the transition preservation constraints \eqref{eq:state_preservation_new} and the reward preservation constraints. Thus according to \Cref{propos:conditional_independent}, the estimated world model will be equivalent to the underlying POMDP.

To efficiently optimize over \eqref{eq:MLE_loss}, we introduce the following posterior to approximate the true posterior distribution: 

\begin{equation}
\label{poster}
    \begin{split}
    q_\psi(b_{s,t},b_{z,t}&|b_{s,t-1}, a_{t-1},b_{z,t-1},o_t)= \\
    &\quad \underbrace{q_\psi(b_{s,t}|b_{s,t-1}, a_{t-1},o_t,b_{z,t})}_{\text{state representation model}}\\ \times 
    & \quad \underbrace{q_\psi(b_{z,t}|b_{z,t-1},o_t,b_{s,t-1}, a_{t-1})}_{\text{noise representation model}}.
    \end{split}
\end{equation}

Consequently, when combining the priors with the structural variational posterior, we derive the following loss obtained from evidence lower bound (ELBO):
\begin{equation}
    \label{eq:eblo}
    \begin{split}
    &\Lcal_\text{ELBO} =
    -\mathbb{E}\bigg[  \underbrace{{\log p(o_t|b_{s,t},b_{z,t})} p_\theta(r_t|b_{s,t},a_{t-1},b_{s,t-1})}_{\text{observation and reward reconstruction loss}} \\
    &-\underbrace{\alpha D_{KL}(q_\psi(b_{s,t}|b_{s,t-1}, a_{t-1},o_t,b_{z,t})\| p_\psi(b_{s,t}|b_{s,t-1},{a_{t-1}}))}_{\text{KL divergence between prior/posterior of state belief}}\\
    &-\underbrace{\beta D_{KL}(q_\psi(b_{z,t}|b_{z,t-1},o_t,b_{s,t-1}, a_{t-1})\| p_\psi(b_{z,t}|b_{z,t-1})) )}_{\text{KL divergence between prior/posterior of noise belief}}\bigg].
    \end{split}
\end{equation}
In our empirical approach, as the dynamics of state belief and noise belief are roughly independent, we employ a mean field approximation to reduce the complexity of our model.  Specifically, we use $p_\psi(b_{z,t}|b_{z,t-1})$ to approximate $p_\psi(b_{z,t}|b_{z,t-1}, b_{s,t})$ . 
We also introduce two hyper-parameters, denoted as $\alpha$ and $\beta$, to effectively enforce the Kullback–Leibler divergence between the variational posterior and prior. It is worth noting that even with the incorporation of $\alpha$ and $\beta$, the modified ELBO continues to serve as an approximate lower bound on the marginal log-likelihood of the observational data in \eqref{eq:MLE_loss}. 
Overall, the proposed RSSM learning objective comprises two key components: the reconstruction loss on reward and observation, and the KL-divergence between prior and posterior. The reconstruction loss on observations, along with KL-divergence encourages the transition preservation constraints in \Cref{eq:state_preservation_new} to be satisfied, and the reconstruction loss on rewards encourages the reward preservation to be satisfied.

\paragraph{Policy Learning}
Similar to the previous works \citep{hafner2020dreamerv2,wang2022denoised}, we also employ online learning to optimize policies through the training of an actor-critic model on the latent state trajectories. These trajectories are exclusively composed of the generated state belief derived from the state belief dynamics model since the absence of noise could improve the sample efficiency of the actor-critic model.

\section{Experiments}

We evaluate the compared methods in diverse image observation contexts, specifically within the DeepMind Control Suite (DMC) \cite{tunyasuvunakool2020dm_control} and RoboDesk \cite{kannan2021robodesk}. The DMC assessment involves six control tasks, namely `Walker Run`, `Cheetah Run`, `Finger Spin`, `Reacher Easy`, `Cartpole Swingup` and `Ball in cup Catch`, with a $64\times64\times3$ pixel observation space, where the methods are tested against environments infused with noisy distractors. The distractor is introduced by replacing the standard blue background with grayscale videos from Kinetics 400's "Driving car" category. In previous methods like TIA \cite{fu2021learning}, DenoisedMDP \cite{wang2022denoised}, and IFactor \cite{liu2023learning}, the same video is picked as the background during all the training episodes. This setting makes the noise easily captured because the noise remains invariant across training episodes. In contrast, in each training episode,
 we involve a randomly selected video as background to mimic the real-world scenario, and further test the methods against unseen noise. As a result, three kinds of DMC environments are compared: \textbf{Noiseless} DMC, DMC with a \textbf{uniform background} (same noise during training), and DMC with \textbf{diverse backgrounds} (changing noise during training).
In the RoboDesk task, with a $96\times96\times3$ pixel observation space, the environment is adapted from DenoisedMDP. Here, the agent's goal is to turn a TV green by pushing a button, with rewards based on the TV's greenness \cite{wang2022robodeskdistractor}. The task is made more challenging by adding noise factors such as Kinetics 400 videos (the same as the DMC setting), environmental flickering, and camera jitter. Training and test noise are kept separate for rigorous testing.
\begin{table*}[h!]
\centering
\resizebox{0.8\textwidth}{!}{%
\small
\begin{tabular}{c|ccccc}
\toprule
\textbf{Task} & \textbf{DreamerV3} & \textbf{TIA} & \textbf{DenoisedMDP} & \textbf{IFactor} & \textbf{Ours}\\ 
\midrule
\textbf{Cartpole Swingup}   & \secondbest{167.15 $\pm$ 54.00}   & 122.28 $\pm$  13.34   & 132.97 $\pm$ 5.44 & 144.72 $\pm$ 27.86 & \best{195.13$\pm$19.00} \\
\textbf{Cheetah Run}       & 144.03 $\pm$ 59.20  & \secondbest{351.91 $\pm$ 127.79} & 221.83 $\pm$ 41.74 & 248.35$\pm$ 104.87 & \best{520.57 $\pm$ 130.84} \\
\textbf{Ball in cup Catch}  & 64.77 $\pm$ 112.19  &31.31 $\pm$ 41.27  & 50.14 $\pm$ 31.08 & \secondbest{134.60 $\pm$ 60.60} & \best{272.60 $\pm$ 136.25} \\
\textbf{Finger Spin}      & 310.66 $\pm$ 121.82 & 282.74 $\pm$ 230.41  & 75.11 $\pm$ 64.41 & \secondbest{566.38 $\pm$ 344.57}& \best{615.01 $\pm$ 211.31}\\
\textbf{Reacher Easy}    & \best{441.33 $\pm$ 472.06} & 86.98 $\pm$ 53.85  & 101.28 $\pm$ 19.60 & \secondbest{176.55 $\pm$ 83.70} &71.43 $\pm$ 13.29\\
\textbf{Walker Run}        & 108.52 $\pm$98.38   & 315.87 $\pm$ 127.00   & 74.36 $\pm$ 19.57  &  68.77 $\pm$ 20.00 & \best{437.20 $\pm$ 71.80}\\ 
\bottomrule
\end{tabular}%
}
\caption{DMC with diverse background results, and the performances were evaluated based on the mean of 10 episode returns and standard deviation from three independent runs. The best results are highlighted in \best{Best Style} and the second best is highlighted in \secondbest{Second Best style}.}
\label{dbg_table}
\end{table*}
\subsection{Baselines}
The proposed method is compared with four model-based baselines: DreamerV3 \cite{hafner2023dreamerv3}, Task Informed Abstractions (TIA) \cite{fu2021learning}, DenoisedMDP \cite{wang2022denoised}, and IFactor \cite{liu2023learning}.
\textbf{DreamerV3}  \cite{hafner2023dreamerv3} the latest version of the Dreamer model, featuring a larger and more robust architecture with a single RSSM optimized for general RL tasks.
\textbf{TIA} \cite{fu2021learning} also designed separate but symmetric dynamics and decoder for state and noise, with the objective of minimizing the reward log-likelihood given noise.
Both \textbf{DenoisedMDP} \cite{wang2022denoised} and \textbf{IFactor} \cite{liu2023learning} seek to achieve this disentanglement by separating the uncontrollable factors, controllable factors, and reward-irrelevant factors in the latent space.
IFactor is also a state-of-the-art reconstruction-based MBRL method that identifies inter-causal relations based on a strong assumption. 
TIA, DenoisedMDP, and IFactor are specifically tailored for operation in noisy environments, whereas DreamerV3 was not designed with a primary focus on denoising. The details of the experiments setting and hyperparameters are provided in Supplementary \ref{sup:hyper}.

\subsection{Performace on DMC tasks}
The performance and reconstruction results in the diverse DMC are presented in Table \ref{dbg_table}, Figure \ref{fig:dbg}, and Table \ref{samples}. These results demonstrate that our method not only ensures optimal disentanglement but also achieves superior performance in policy learning. Similarly, our method also exhibits exceptional performance compared with other approaches in most tasks in Uniform DMC (details in Table \ref{ubg_table} and Figure \ref{fig:ubg_chart}), where disentangling is easier than that in Diverse DMC. In noiseless DMC tasks, our method shows comparable performance when measured against the SOTA performance, detailed in Table \ref{noiseless_table} and Figure \ref{fig:noiseless_chart}.

\paragraph{Noise belief matters}
In a comparative analysis, DreamerV3 underperforms in most tasks, even with its largest scale model, highlighting structural limitations in its single transition and decoder framework, which is inadequate for effective denoising. In contrast, our method, which models state belief and noise belief separately, excels by better preserving transition dynamics. Our noise belief RSSM not only accounts for noise transitions but also captures the inherent uncertainty within the system. This is evident in the reconstruction results in Table \ref{samples}, where reconstructions of noise belief lack the detail seen in state belief reconstructions. This discrepancy likely arises from the complex and uncertain nature of latent noise. By effectively managing both noise and system uncertainty within the noise belief RSSM, our method achieves a clearer and more distinct state representation from noisy observations. The results underscore the advantages of separating noise and state dynamics into distinct models, leading to more precise and dependable state estimation in environments characterized by substantial observational noise and uncertainty.

\begin{table*}[h!]
\centering
\resizebox{\textwidth}{!}{%
\begin{adjustbox}{valign=c}
\begin{tabular}{cccccc}
\hline
\textbf{ }&\textbf{Dreamer} & \textbf{TIA} & \textbf{DenoisedMDP}& \textbf{IFactor} &\textbf{Ours} \\
\hline
\textbf{Raw Observation} & \includegraphics[width=0.2\linewidth,height=0.3\linewidth]{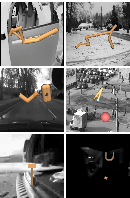} &\includegraphics[width=0.2\linewidth,height=0.3\linewidth]{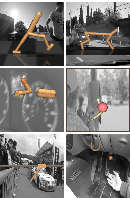}&\includegraphics[width=0.2\linewidth,height=0.3\linewidth]{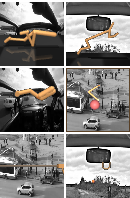}&\includegraphics[width=0.2\linewidth,height=0.3\linewidth]{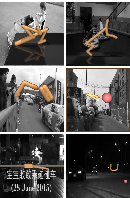}&\includegraphics[width=0.2\linewidth,height=0.3\linewidth]{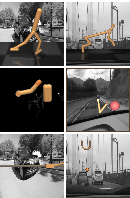}  \\
\hline
\textbf{State Reconstruction} & \includegraphics[width=0.2\linewidth,height=0.3\linewidth]{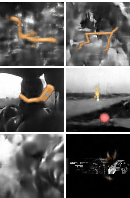}&\includegraphics[width=0.2\linewidth,height=0.3\linewidth]{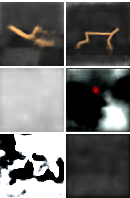}&\includegraphics[width=0.2\linewidth,height=0.3\linewidth]{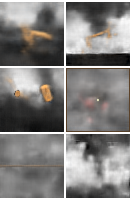}&\includegraphics[width=0.2\linewidth,height=0.3\linewidth]{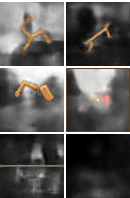}&\includegraphics[width=0.2\linewidth,height=0.3\linewidth]{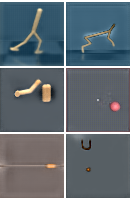}  \\ \hline
\textbf{Noise Reconstruction} & &\includegraphics[width=0.2\linewidth,height=0.3\linewidth]{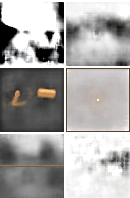}&\includegraphics[width=0.2\linewidth,height=0.3\linewidth]{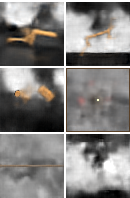}&\includegraphics[width=0.2\linewidth,height=0.3\linewidth]{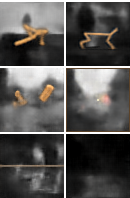}&\includegraphics[width=0.2\linewidth,height=0.3\linewidth]{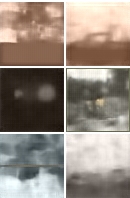}  \\
\hline
\end{tabular}
\end{adjustbox}
}
\caption{The reconstruction visualization of 6 DMC tasks with diverse backgrounds, demonstrates the reconstruction of state and noise from different world models. Our models exhibit clear decoupling between the state belief and the noise belief, whereas other compared methods produce numerous erroneous reconstructions and consistently sacrifice the state observation.}
\label{samples}
\end{table*}

\paragraph{Transition preservation and reward preservation are sufficient}
By maintaining transition preservation and reward preservation throughout the reconstruction and KL objectives, our method demonstrates superior performance across most tasks, as detailed in Tables \ref{dbg_table} and \ref{samples}, and Figure \ref{fig:dbg}. The reconstructed samples from the learned beliefs of state and noise highlight our method’s superior disentanglement capabilities compared to other denoising techniques, particularly in tasks like 'Cartpole Swingup' and 'Ball in Cup Catch.' Here, our method excels in separating clear state information from noise, outperforming competitors that capture minimal state details, especially in robotic tasks. Additionally, our approach leads in policy learning, securing the best performance in 5 of 6 tasks, as shown in Table \ref{dbg_table} and Figure \ref{fig:dbg}. The exception is the Reacher Easy task, where DreamerV3, not tailored for denoising, performs best. These results underscore that even without detailed identification of each latent state component, our approach achieves effective disentanglement between state and noise, facilitating the development of an optimal policy under mild assumptions.

\subsection{Performance on Robodesk}
Given the provision of Robodesk, which offers a comprehensive reward-related state belief observation (with a high signal-to-noise ratio), such as assessing the greenness of the TV screen, we find that the performance of the compared methods is closely matched. Therefore, the performance in this tailored environment may not reflect the true disentanglement ability of the compared method. The corresponding results are presented in Supplementary \ref{robodesk} and Figure \ref{fig:robodesk_chart}
.

\subsection{Ablation Study}
\begin{table*}[h!]
\resizebox{\textwidth}{!}{%
\begin{tabular}{c|cccccc}
\toprule
\textbf{Task} & \textbf{Cartpole Swingup} & \textbf{Cheetah Run} & \textbf{Ball in cup Catch} & \textbf{Finger Spin} & \textbf{Reacher Easy} & \textbf{Walker Run}\\ 
\midrule
\textbf{Asymetric reconstructions}   &  \best{195.13 $\pm$ 19.00}   & \best{520.57 $\pm$ 130.84}   & \best{272.60 $\pm$ 136.25} & \best{615.01 $\pm$ 211.31} & 71.43 $\pm$ 13.29 & \best{437.20 $\pm$ 71.80}\\
\textbf{Symetric reconstructions}       & 141.94 $\pm$ 10.85  & 292.49 $\pm$ 94.19 & 126.59 $\pm$ 35.30 &269.87 $\pm$ 218.16 & \best{90.41 $\pm$ 22.43} &320.56 $\pm$ 247.30\\
\textbf{WO reward}   & 87.88 $\pm$ 33.96   & 57.80 $\pm$ 71.91   & 24.62 $\pm$40.95 & 1.37 $\pm$ 1.36 & 67.95 $\pm$ 17.83 & 23.89 $\pm$ 16.62\\
\textbf{WO KL divergence}       & 107.41 $\pm$26.32  & 20.36 $\pm$  8.42 & 51.28 $\pm$  19.20 & 3.71 $\pm$ 0.84 & 78.15 $\pm$ 16.65 &29.57 $\pm$ 15.74\\
\bottomrule
\end{tabular}%
}
\caption{Abaltion study results. 4 configurations of the world model were tested on Diverse DMC tasks: asymmetric reconstruction, symmetric reconstruction, no reward objective, and no KL divergence objective. With asymmetric reconstruction, our proposed method, outperformed the others in nearly all tasks except Reacher Easy, with the most significant outcomes highlighted in \best{Best Style}.}
\label{ablation_table}
\end{table*}
\textbf{Asymmetric Reconstructions, Reward Objective, and KL Divergence Objective Matter}
In the ablation study, we assess performance differences among symmetric reconstructions, asymmetric reconstructions, and models omitting reward and KL divergence objectives, using the DeepMind Control (DMC) suite with diverse backgrounds. The symmetric setup utilized dual 8-layer decoders, whereas the asymmetric employed an 8-layer state belief decoder alongside a 4-layer noise belief decoder. Results in Table \ref{ablation_table} and Figure \ref{fig:ablation}in the supplementary materials illustrate the superior disentanglement provided by asymmetric reconstructions. This setup reduces the likelihood of confusing noise belief with state belief, unlike symmetric reconstructions. As further evidenced in Table \ref{dbg_table}, Table \ref{samples}, and Figure \ref{fig:dbg}, the asymmetric reconstruction significantly boosts disentanglement efficiency, a key distinction from the TIA model. This confirms that separated asymmetric structures enhance transition preservation. Additionally, models lacking either the reward objective or the KL divergence objective demonstrate compromised functionality, underscoring the essential nature of these objectives in ensuring effective transition and reward preservation.

\section{Conclusion}

In this work, we revisited the challenge of latent state estimation in RL under noisy conditions, focusing on the role of causal identifiability in ensuring the reliable recovery of latent states. While previous studies have relied on assumptions that can be impractical within the RL context, such as the necessity for independent subsets of latent states and invertible observation functions, we demonstrated that these assumptions could be significantly relaxed. By incorporating RL-specific perspectives—particularly through leveraging rewards and transitions—we established a more general approach to identifiability in partially observable Markov decision processes (POMDPs). Our theoretical findings show that the optimal recovery of latent states requires neither independence among latent state subsets nor invertibility of the observation function.

Furthermore, we proposed a novel methodology for general POMDPs that integrates two simple yet effective constraints: transition reservation and reward reservation. This methodology guarantees the disentanglement of state and noise in a manner faithful to the underlying dynamics, which we validated empirically across extensive benchmark control tasks. The results highlight the efficacy of our approach, outperforming existing algorithms by ensuring more effective disentanglement of state belief from noise belief. By bridging the gap between causal theory and practical RL applications, this work opens new possibilities for designing more robust RL algorithms, offering a foundation for future advancements in RL under uncertainty.


\bibliography{main}
\clearpage
\clearpage
\appendix
\section{Reproducibility Checklist}
This paper:

Includes a conceptual outline and/or pseudocode description of AI methods introduced (yes)\\
Clearly delineates statements that are opinions, hypothesis, and speculation from objective facts and results (yes)\\
Provides well marked pedagogical references for less-familiare readers to gain background necessary to replicate the paper (yes)\\
Does this paper make theoretical contributions? (yes)\\

If yes, please complete the list below.

All assumptions and restrictions are stated clearly and formally. (yes)\\
All novel claims are stated formally (e.g., in theorem statements). (yes)\\
Proofs of all novel claims are included. (yes)\\
Proof sketches or intuitions are given for complex and/or novel results. (yes)\\
Appropriate citations to theoretical tools used are given. (yes)\\
All theoretical claims are demonstrated empirically to hold. (yes)\\
All experimental code used to eliminate or disprove claims is included. (yes)\\
Does this paper rely on one or more datasets? (yes)\\

If yes, please complete the list below.

A motivation is given for why the experiments are conducted on the selected datasets (yes)\\
All novel datasets introduced in this paper are included in a data appendix. (no)\\
All novel datasets introduced in this paper will be made publicly available upon publication of the paper with a license that allows free usage for research purposes. (yes)\\
All datasets drawn from the existing literature (potentially including authors’ own previously published work) are accompanied by appropriate citations. (yes)\\
All datasets drawn from the existing literature (potentially including authors’ own previously published work) are publicly available. (yes)\\
All datasets that are not publicly available are described in detail, with explanation why publicly available alternatives are not scientifically satisficing. (yes)\\
Does this paper include computational experiments? (yes)\\

If yes, please complete the list below.

Any code required for pre-processing data is included in the appendix. (yes).\\
All source code required for conducting and analyzing the experiments is included in a code appendix. (yes)\\
All source code required for conducting and analyzing the experiments will be made publicly available upon publication of the paper with a license that allows free usage for research purposes. (yes)\\
All source code implementing new methods have comments detailing the implementation, with references to the paper where each step comes from (yes)\\
If an algorithm depends on randomness, then the method used for setting seeds is described in a way sufficient to allow replication of results. (yes)\\
This paper specifies the computing infrastructure used for running experiments (hardware and software), including GPU/CPU models; amount of memory; operating system; names and versions of relevant software libraries and frameworks. (yes)\\
This paper formally describes evaluation metrics used and explains the motivation for choosing these metrics. (yes)\\
This paper states the number of algorithm runs used to compute each reported result. (yes)\\
Analysis of experiments goes beyond single-dimensional summaries of performance (e.g., average; median) to include measures of variation, confidence, or other distributional information. (yes)\\
The significance of any improvement or decrease in performance is judged using appropriate statistical tests (e.g., Wilcoxon signed-rank). (yes)\\
This paper lists all final (hyper-)parameters used for each model/algorithm in the paper’s experiments. (yes)\\
This paper states the number and range of values tried per (hyper-) parameter during development of the paper, along with the criterion used for selecting the final parameter setting. (yes)\\

\clearpage

\section{Proof of Propositions}
\setcounter{theorem}{0}
\begin{theorem}[For Invertible $\Mcal$]\label{propos:Invertible}
 Given POMDP $(\mathcal{S},\mathcal{A},\mathcal{O},\mathcal{T},\mathcal{M},\Rcal, \gamma)$, if the observation function $\Mcal$ is invertible. 
    Let $g:\Ocal\mapsto \Scal\times\Zcal$ be an invertible state estimation function. If the following condition holds:
    \begin{enumerate}
    \item \textbf{No redundancy: } For different states $s_1$ and $s_2 \in \Scal$ , there exists an policy $\pi$ that $V^\pi(s_1)\neq V^\pi(s_2)$. 
    \item \textbf{Transition preservation: } For any observation $o$ and any action $a$, let $\hat{s}, \hat{z} = g(o)$, we have
    \begin{equation}
        p(o'|a,o)=p(\hat{s}'|a,\hat{s})p(\hat{z}'|\hat{z})
    \end{equation}
    \item \textbf{Reward preservation:} For any state observation pair $(s, o)$ with $\hat{s}, \hat{z} = g(o)$, for any action $a$, we have
    \begin{equation}
        \exists\hat{\Rcal} \text{ s.t. } \Rcal(o', a,o)=\Rcal(s', a,o)=\hat{\Rcal}(\hat{s}',a,\hat{s})
    \end{equation}
  \end{enumerate}
  then the estimated MDP $(\hat{\Scal}=\{\hat{s}|\exists o, \text{s.t. } \hat{s}, \hat{z} = g(o) \}, \Acal, \hat{\Tcal} = p(\hat{s}'|a, \hat{s}), \hat{\Rcal}, \gamma) $ is equivalent to the underlying MDP $(\mathcal{S},\mathcal{A},\mathcal{T},\Rcal, \gamma)$, and the state estimation function $g$ disentangles state and noise. 
\end{theorem}
\begin{proof}
  For any policy $\pi: \Scal \mapsto \Acal$,   $\forall s \in \Scal$ and for any $o$ that satisfies $s,z=\Mcal(o)$, we can define a map $\tilde\Mcal: s = \tilde\Mcal(o)$
  \begin{subequations}\label{eq:derivation_v_equ}
    \begin{align}
        V^{\pi}(s) =& \sum_{s'}p(s'|\pi(s),s)[\gamma V^{\pi}(s')+\Rcal(s',\pi(s),s)] \notag \\ 
                   =&  \sum_{s', z'}\underbrace{p(s'|\pi(s), s)p(z'|z)}_{\text{By the causal graph \Cref{generative model}}}[\gamma V^{\pi}(\tilde{\Mcal}(o')) \notag \\
                    &+ \Rcal(\tilde{\Mcal}(o'), \pi \circ \tilde\Mcal (o),\tilde{\Mcal}(o))] \\
                   = & \sum_{o'}p(o'|\pi\circ\tilde{\Mcal}(o), o)[\gamma V^{\pi\circ\tilde{\Mcal}}(o') \\ 
                    &+ \Rcal(\tilde{\Mcal}(o'), \pi \circ \tilde\Mcal (o),\tilde{\Mcal}(o))]\notag \\ 
                   = & V^{\pi\circ\tilde{\Mcal}}(o),
    \end{align}
  \end{subequations}
  which indicates that the underlying MDP is value equivalent to an MDP that uses observation as a state. Similarly, for the estimated MDP, for any policy $\hat{\pi}: \hat{\Scal}\mapsto \Acal$, $\forall \hat{s} \in \hat{\Scal}$ 
    and for any $o$ that satisfies $\hat{s},\hat{z}=g(o)$, we can define a map $\tilde g: \hat s = \tilde{g}(o)$
    Then  consider  $\hat{s},\hat{z}=g(o)$, we have
    \begin{subequations}\label{eq:inv_v_equ}
    \begin{align}
        V^{\hat{\pi}}(\hat{s}) = & \sum_{\hat s'}p(\hat s'|\hat\pi(\hat s),\hat s)[\gamma V^{\hat \pi}(\hat s')+\hat \Rcal(\hat s',\hat \pi(\hat s),\hat{s})] \notag \\ 
         =&  \sum_{\hat s', \hat z'}\underbrace{p(\hat{s}'|\hat{\pi}(\hat s), \hat s)p(\hat{z}'|\hat z)}_{\text{By the transition preservation condition}}[\gamma V^{\hat \pi}(\tilde{g}(o')) \notag \\
          &+ \underbrace{\Rcal(\tilde{g}(o'), \hat{\pi }\circ \tilde{g} (o),\tilde{g}(o))}_{\text{By reward preservation condition}}]\\
                   = & V^{\hat\pi\circ\tilde{g}}(o).
    \end{align}
  \end{subequations}
  Then the estimated MDP must be equivalent to the underlying MDP. Then consider the mapping $f:\hat{\Scal}\mapsto\Scal$, if it maps $\hat{s}$ to multiple $s\in\Scal$, then for some policy $\hat{s}$ must have more than one value, which is in contradict with the fact that for any MDP, with given policy the value of a state is deterministic. Similar, by the No redundancy condition $f^{-1}$ can not map $s\in\Scal$ to multiple $\hat{\Scal}$. 

    In the above proof, we proved the theorem in discrete cases, but it would be straightforward to extend the result to continuous state and observation spaces. 
  \end{proof}

  \begin{repproposition}{propos:conditional_independent}
    Under the assumptions of Proposition \ref{propos:disentangle}, the transition preservation condition in Proposition \ref{propos:disentangle} can be extended as $\hat{z}$ is conditional independent with future state $\hat{s}'$ given $\hat{s}$, and the results of Proposition \ref{propos:disentangle} still hold.
  \end{repproposition}
  \begin{proof}
    Consider the decomposition of joint distribution
    \[
      p(s', z'|a, s, z)
    \]
    under the constraints that $z$ is conditional independent with $s'$ given $s$ as well as action $a$, the possible decompsitions are
    \begin{subequations}
      \begin{align}
        p(s', z'|a, s, z) & = p(s'|a, s)p(z'|z),\\
        p(s', z'|a, s, z) & = p(s'|a, s)p(z'|a, z),\\
        p(s', z'|a, s, z) & = p(s'|a, s)p(z'|a, z, s),\\
        p(s', z'|a, s, z) & = p(s'|a, s)p(z'|a, z, s'),\\
        p(s', z'|a, s, z) & = p(s'|a, s)p(z'|a, z, s, s').
      \end{align}
    \end{subequations}
    For each decomposition, the derivation in \Cref{eq:derivation_v_equ} still holds, and \Cref{eq:inv_v_equ} also holds, thus the results in \Cref{propos:Invertible} still holds. 
  \end{proof}

\section{Generative model}
Our theoretical contributions apply to a wide range of tasks with different generative graphs e.g. in Figure \ref{generative model}, which elucidates the role of actions as auxiliary variables from a causal standpoint. In the RL framework, modifications in the policy directly affect the distribution of cumulative rewards, a relationship that remains unaffected by variations in the noise. This delineation establishes the foundation for our investigation, which aims to devise methodologies for effectively distinguishing and separating the states from the noise under minimal assumptions. By capitalizing on the Markovian properties inherent to Markov Decision Processes (MDPs), we contend that assessing whether changes in the state influence the expected return is equivalent to analyzing the transition preservation and reward preservation.

\begin{figure}[t]
\centering
\includegraphics[width=0.3\textwidth]{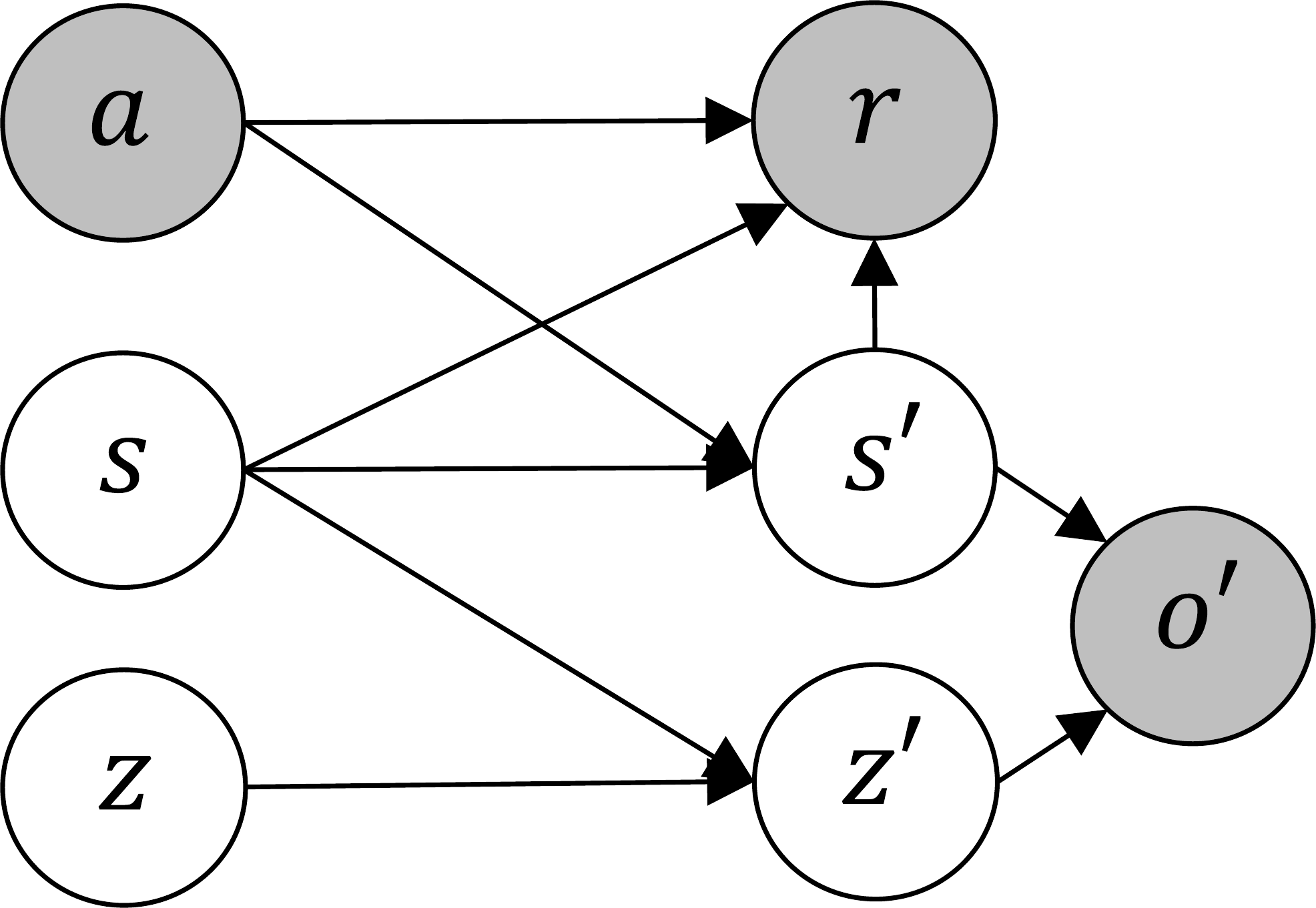}
\caption{This is a typical generative model for RL on MDP with the noisy observation, where  $s$ denotes the latent state while $z$ represents the latent noise, both of which remain unobservable. Previous work \cite{huang2022action,liu2023learning} is built on different generative models, e.g. dividing $s$ into independent parts with additional assumptions, which could hinge the application in the real world. In contrast, our work is not limited by a specific generative model and guarantees that the estimated MDP is equivalent to the underlying true MDP when transition preservation and reward preservation are maintained. 
}
\label{generative model}
\end{figure}

\section{Perfomance on DMC vatiants}
\label{performance}
\subsection{Noiseless DMC}
The performance of the compared method in Noiseless DMC environments is illustrated in  \Cref{noiseless_table} and Figure \ref{fig:noiseless_chart}. In these noiseless environments, symmetric decoders are employed for both state belief and noise belief. Due to the minimal presence of noise, the raw observation can be considered a state belief observation, rendering only the state belief decoder necessary. If we were to retain only one decoder, the proposed structure would be reduced to a typical dreamer model.
\begin{table*}[h]
\centering
\resizebox{0.9\textwidth}{!}{%
\begin{tabular}{c|ccccc}
\hline
\textbf{Task} & \textbf{DreamerV3} & \textbf{TIA} & \textbf{DenoisedMDP} & \textbf{IFactor} & \textbf{Ours} \\ 
\hline
\textbf{Cartpole Swingup}  & \secondbest{824.62 $\pm$ 59.36}   & 716.55 $\pm$ 151.01 & 555.49 $\pm$ 416.89  & \best{836.21 $\pm$ 26.45} &824.27 $\pm$ 12.01 \\
\textbf{Cheetah Run}       & 780.62 $\pm$ 106.69 & 573.56 $\pm$ 361.39  & \best{815.78 $\pm$ 72.20}  & 717.14 $\pm$ 213.07 & \secondbest{797.16 $\pm$ 24.96}\\
\textbf{Ball in cup Catch} & \best{970.44 $\pm$ 9.04} & \secondbest{959.00 $\pm$ 6.38}   & 956.48 $\pm$ 8.21  & 908.54 $\pm$ 102.09 &947.77 $\pm$ 20.35\\
\textbf{Finger Spin}       &\best{939.22 $\pm$ 22.22} & 252.66 $\pm$ 219.10  & 442.76 $\pm$ 152.80 & \secondbest{695.97 $\pm$ 248.07} &469.78 $\pm$ 228.85\\
\textbf{Reacher Easy}   & 801.44 $\pm$ 166.39  & 891.66 $\pm$ 120.26 & 680.30 $\pm$ 259.33 & \secondbest{948.92 $\pm$ 39.07} & \best{965.88 $\pm$ 17.50} \\
\textbf{Walker Run}    & \best{762.29 $\pm$ 103.05} & 553.41 $\pm$ 55.30   & 629.54 $\pm$ 57.89  & 405.77 $\pm$ 78.61 & \secondbest{642.19 $\pm$ 39.93} \\\hline
\end{tabular}%
}
\caption{The performances of the compared methods in the Noiseless DMC environments. The The best results are highlighted in \best{Best Style} and the second best is highlighted in \secondbest{Second Best style}. Our method shows a comparable performance against the SOTA performance in most tasks.}
\label{noiseless_table}
\end{table*}

\begin{figure*}[h]
    \centering
    \includegraphics[width=\linewidth]{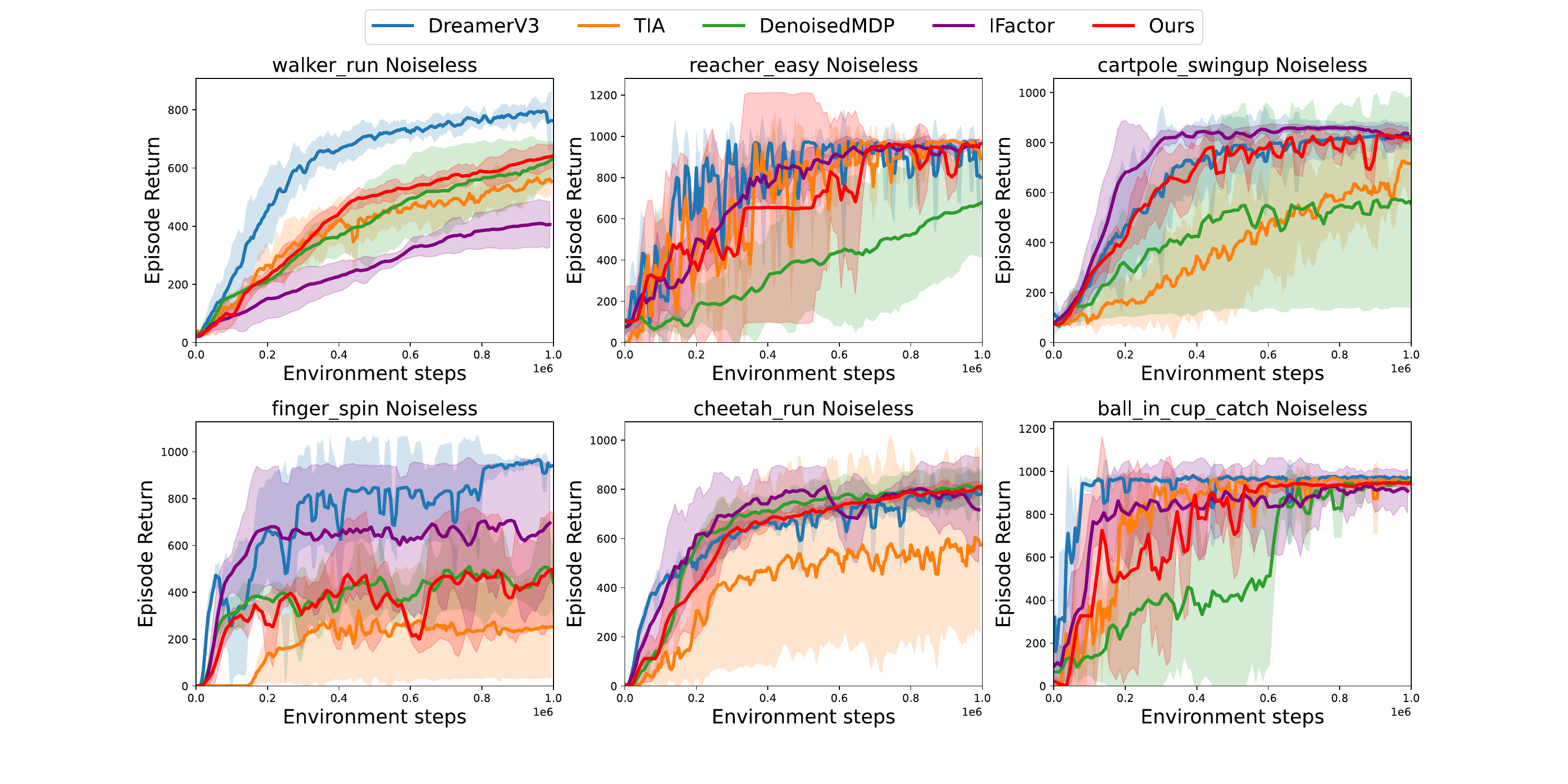}
    \caption{Our method is specifically designed to disentangle latent state from latent noise, which means it may not perform exceptionally in noiseless environments. However, evaluating performance in noiseless settings can offer valuable insights by allowing a comparative analysis of each method's effectiveness in both noisy and noiseless environments.}
    \label{fig:noiseless_chart}
\end{figure*}

\subsection{DMC with uniform background}
In a uniform background scenario, the video background is sampled only once during training, causing all training episodes to share the same noisy background. As a result, a potential shortcut for achieving disentanglement is to remove the invariant noise rather than preserving the true state. However, this approach is not feasible in real-world applications where such controlled conditions do not exist. All the previous work (TIA, DenoisedMDP and IFactor) involved training and testing their method on the Uniformed background DMCs, where the disentanglement could be easier. The performance of the compared methods is shown in Table \ref{ubg_table} and Figure \ref{fig:ubg_chart}. Since DreamerV3 is not designed for denoising, it is not included in the comparison on the Uniform DMC.


\begin{table*}[h!]
\centering
\resizebox{0.8\textwidth}{!}{%
\begin{tabular}{c|cccc}
\hline
\textbf{Task} & \textbf{TIA} & \textbf{DenoisedMDP} & \textbf{IFactor} & \textbf{Ours} \\ 
\hline
\textbf{Cartpole Swingup}   &119.94 $\pm$  16.41  & 97.19 $\pm$ 18.58   & \best{209.80 $\pm$ 354.54} & \secondbest{178.86 $\pm$ 46.09} \\
\textbf{Cheetah Run}       & 291.08 $\pm$ 74.14  &  317.31 $\pm$ 13.66  & \best{514.51 $\pm$ 165.56} & \secondbest{469.47 $\pm$ 117.29}\\
\textbf{Ball in cup Catch}  & 52.36 $\pm$ 49.66  & \secondbest{120.22 $\pm$ 25.17}  & 5.63 $\pm$ 9.75 & \best{201.81 $\pm$ 156.10}\\
\textbf{Finger Spin}      & 354.44 $\pm$ 299.01 &  \secondbest{559.95 $\pm$ 47.03}  & 504.3 $\pm$ 169.35 & \best{603.84 $\pm$ 231.17} \\
\textbf{Reacher Easy}    & 366.19 $\pm$ 129.51  &639.53 $\pm$ 118.87  &  \secondbest{832.36 $\pm$ 79.60}& \best{971.75 $\pm$ 27.06}\\
\textbf{Walker Run}        & 325.18 $\pm$ 42.01   & \secondbest{401.35 $\pm$ 59.74}   & 212.36 $\pm$ 166.29  & \best{455.15 $\pm$ 137.68}\\ \hline
\end{tabular}%
}
\caption{The performance of the compared methods in Uniform DMC environments with unseen video backgrounds is presented below. The best results are highlighted in \best{Best Style} and the second best is highlighted in \secondbest{Second Best style}. Our method achieves the best performance in 4 out of 6 tasks and demonstrates comparable results in the other 2 tasks.}\label{ubg_table}
\end{table*}

\begin{figure*}[h]
    \centering
    \includegraphics[width=1\linewidth]{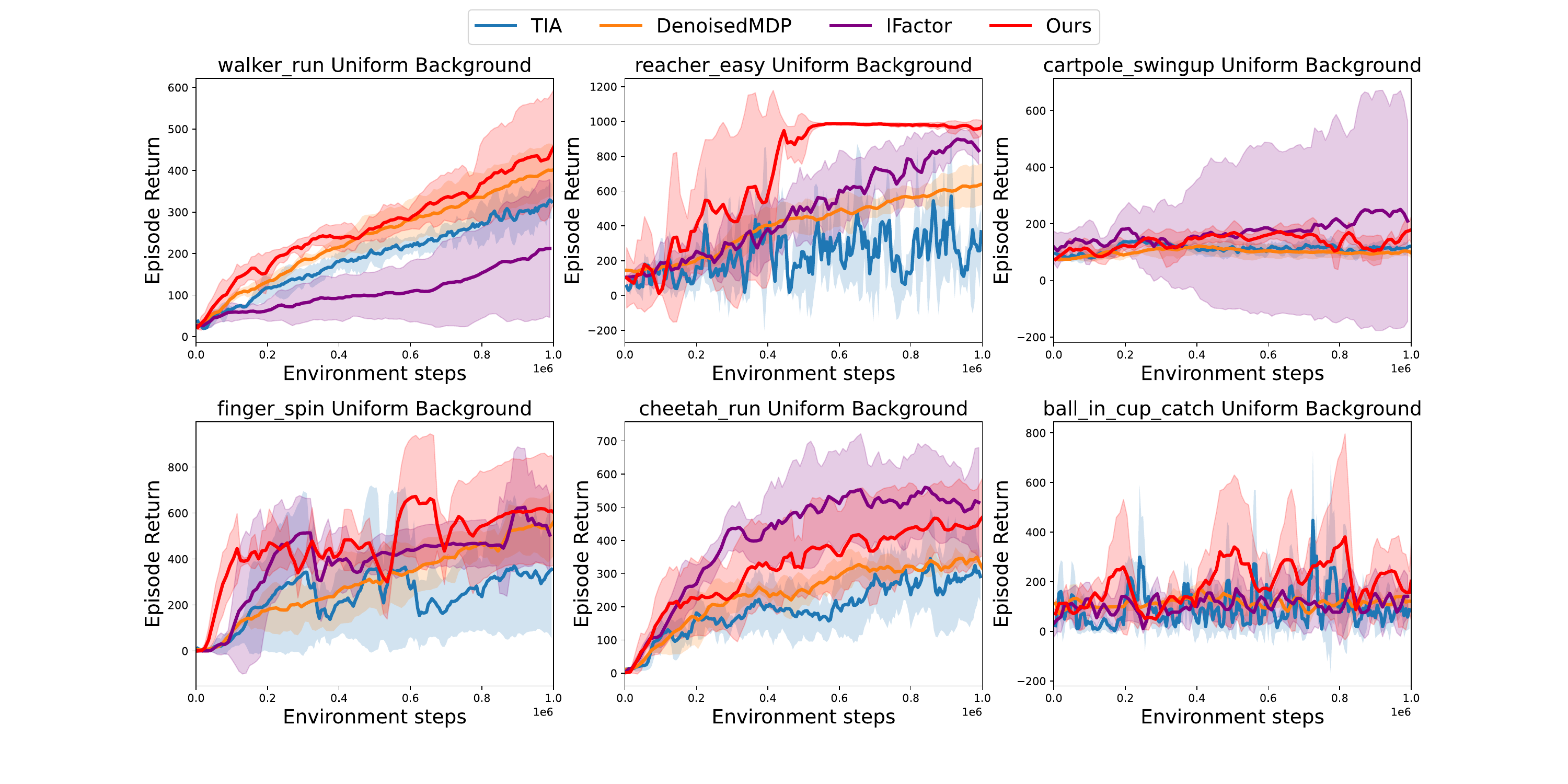}
    \caption{In the DMC with a uniform background, where the environment setting is less complex than in diverse scenarios, our method either achieves the best performance or consistently demonstrates comparable performance when measured against SOTA results.}
    \label{fig:ubg_chart}
\end{figure*}
\subsection{DMC with diverse background}
\paragraph{Disentanglement reflected on the observation space}
In assessing state belief reconstruction capabilities as a measure of disentanglement proficiency, our method, as detailed in Table \ref{samples}, demonstrates superior disentanglement across various tasks. The ability to distinguish the state belief within an environment is directly proportional to the robot's observable body proportion in the total environment, similar to a signal-to-noise ratio (SNR). High SNR tasks, such as Walker Run, Cheetah Run, and Finger Spin, offer simpler state belief identification. Conversely, low SNR tasks like Reacher Easy, Cartpole Swingup, and Ball in Cup Catch, present more significant challenges in state belief identification.
Comparative analysis indicates that alternative methods, such as TIA, DenoisedMDP, and IFactor, encounter difficulties in preserving state belief integrity and differentiating state belief from noise belief. This issue is particularly pronounced in high SNR tasks like `Walker Run` and `Cheetah Run`, where these methods tend to conflate state belief elements with the changing noise belief. The challenge is amplified in low SNR environments, diminishing the efficacy of these methods in accurately identifying the robot in tasks like `Ball in Cup Catch` and `Cartpole Swingup`.
In contrast, our method effectively separates state belief from noise belief in both high and low SNR tasks. For instance, in the `Ball in Cup Catch` task, it distinctly identifies both the cup and ball; in `Cartpole Swingup`, it captures the cart accurately; and in `Reacher Easy`, it identifies the red destination. This enhanced state belief reconstruction ability leads to improved overall performance, as outlined in Table \ref{dbg_table}. The only scenario where our method is outperformed is in the Reacher Easy task, where DreamerV3, with its larger architecture, showed superior results. Specifically, DreamerV3 utilizes a state dimension of 512, a posterior dimension of 32, and an observation embedding dimension of 4096. In contrast, our model employs a belief dimension of 120, a posterior dimension of 20, and an observation embedding dimension of 1024. We also conducted a simple ablation study on the Reacher Easy task using a belief dimension of 512 and a posterior dimension of 32, where we achieved a comparable result of $313.4 \pm 217.89$.
\begin{figure*}[h!]
    \centering
    \includegraphics[width=1\linewidth]{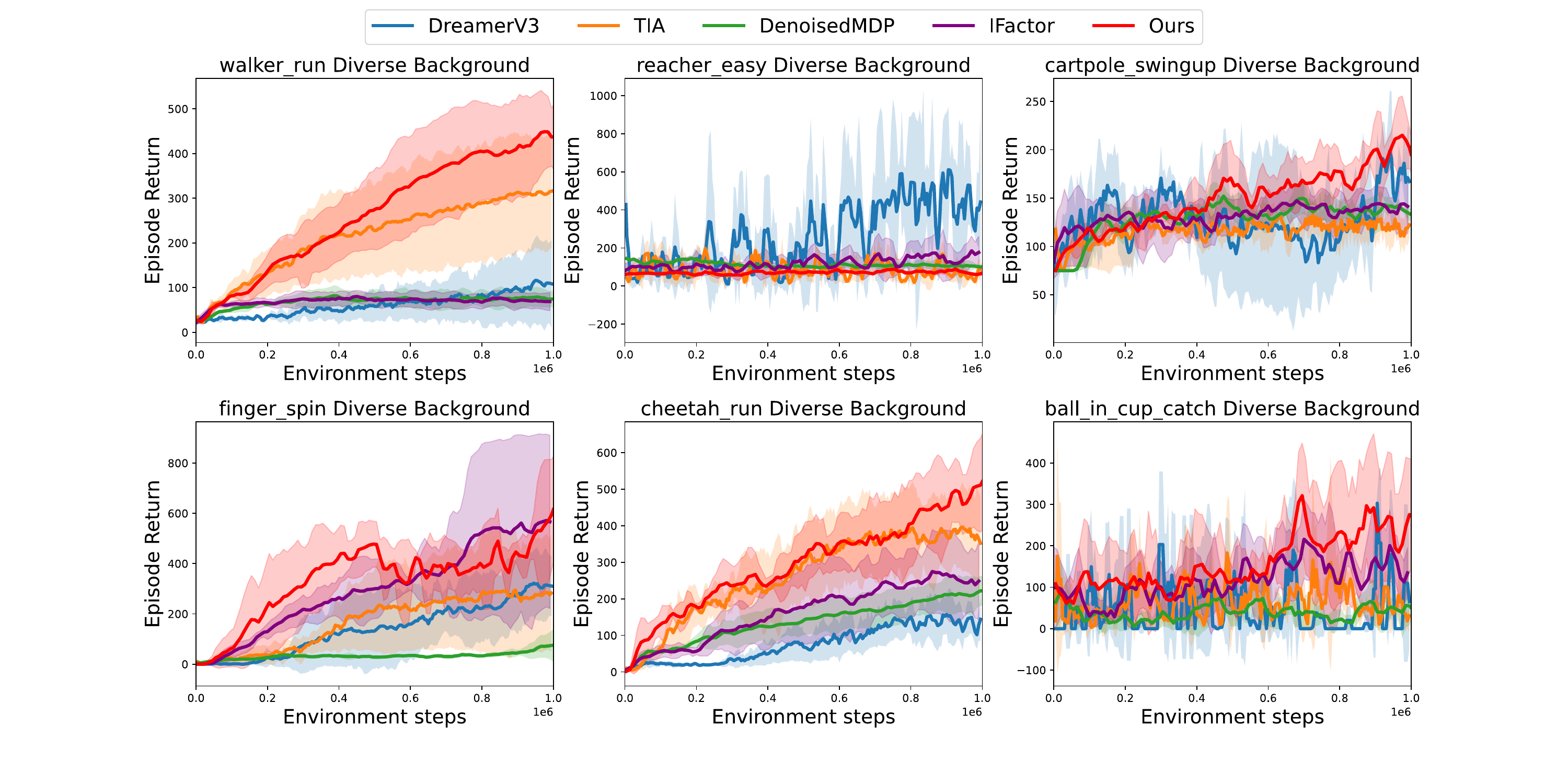}
    \caption{In the diverse video background scenario, our method outperforms strong baselines and achieves the best performance in 5 out of 6 tasks, except Reacher Easy. DreamerV3, benefiting from its larger architecture, achieves the best performance on Reacher Easy. Our method could also achieve comparable performance with an expanded belief space and posterior space. }
    \label{fig:dbg}
\end{figure*}
\section{Performance on Robodesk}
\label{robodesk}
In the Robodesk tasks, we can see that all the compared methods achieve quite close performance, around 500. However, because of the rich reward-related state belief (SNR is large), the error tolerance is much higher than DMC. Therefore, Robodesk might not be an appropriate environment for the disentanglement evaluation. 

\begin{figure*}[h]
    \centering
\includegraphics[width=0.7\linewidth]{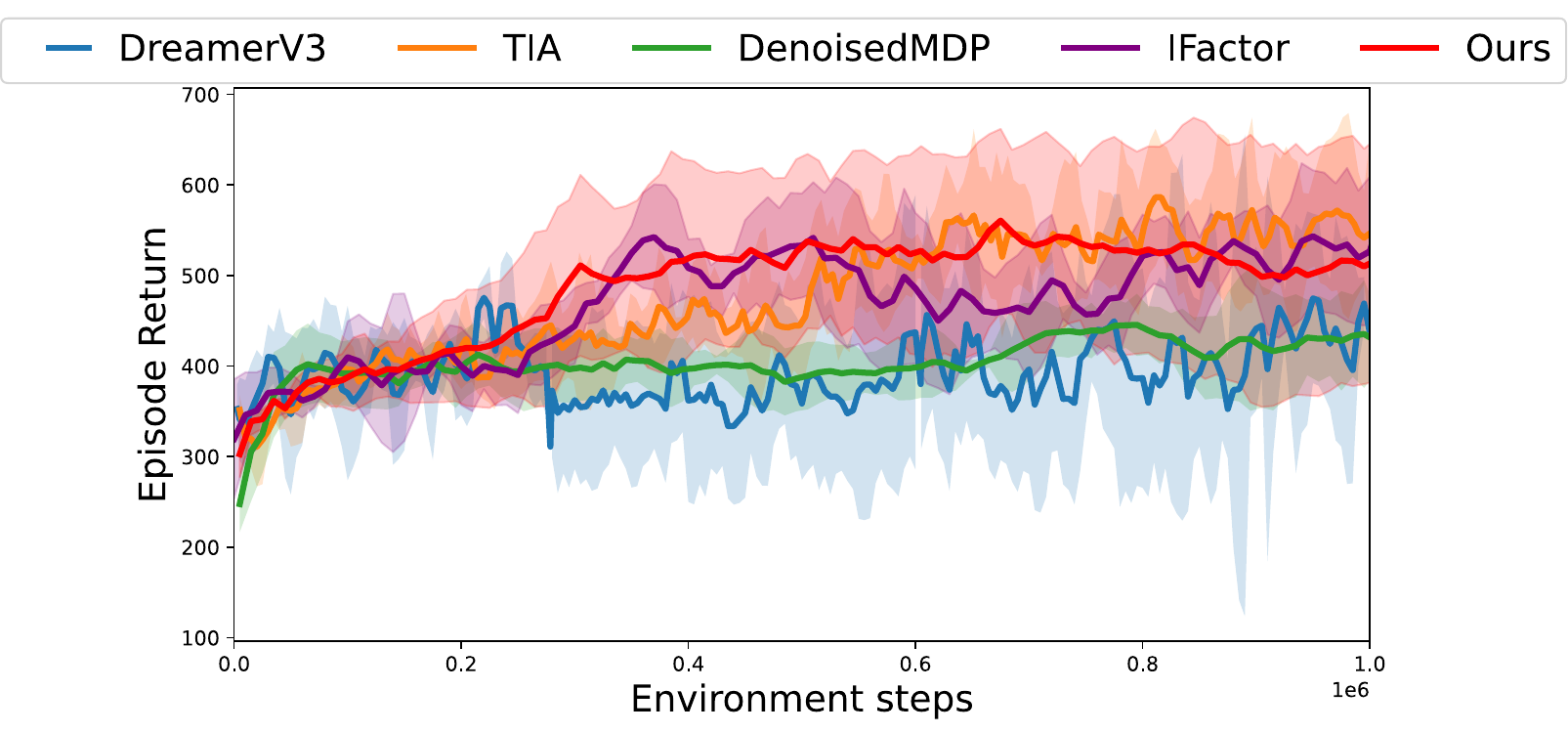}
    \caption{All compared methods demonstrate comparable performance on RoboDesk.}
    \label{fig:robodesk_chart}
\end{figure*}

\begin{table*}[h]
\centering
\resizebox{\textwidth}{!}{%
\begin{adjustbox}{valign=c}
\begin{tabular}{cccccc}
\hline
\textbf{ }&\textbf{DreamerV3} & \textbf{TIA} & \textbf{DenoisedMDP}& \textbf{IFactor}& \textbf{Ours} \\
\hline
\textbf{Raw Observation} & \includegraphics[width=0.15\linewidth]{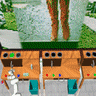} &\includegraphics[width=0.15\linewidth]{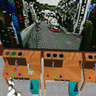}&\includegraphics[width=0.15\linewidth]{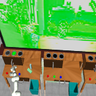}&\includegraphics[width=0.15\linewidth]{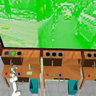} &\includegraphics[width=0.15\linewidth]{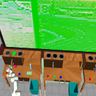}  \\
\hline
\textbf{State Reconstruction} & \includegraphics[width=0.15\linewidth]{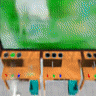}&\includegraphics[width=0.15\linewidth]{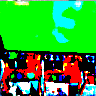}&\includegraphics[width=0.15\linewidth]{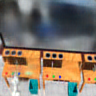}&\includegraphics[width=0.15\linewidth]{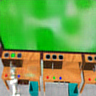}&\includegraphics[width=0.15\linewidth]{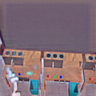}  \\ \hline
\textbf{Noise Reconstruction} & &\includegraphics[width=0.15\linewidth]{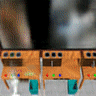}&\includegraphics[width=0.15\linewidth]{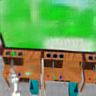}&\includegraphics[width=0.15\linewidth]{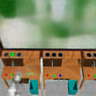} &\includegraphics[width=0.15\linewidth]{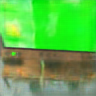} \\
\hline
\end{tabular}
\end{adjustbox}%
}
\caption{In the Robodesk task, all methods exhibit comparable performance, but their denoising capabilities vary. TIA fails to identify the robot arm; Even DenoisedMDP mistakes the greenness as noise; IFactor also mistakes the green light for noise. Although methods confuse the state and noise, they still get quite good performance maybe because the reward-related state observation is rich.}
\label{samples_robodesk}
\end{table*}

\section{Ablation study}
Table \ref{ablation_table} and Figure \ref{fig:ablation} present the findings from our ablation study. This study examines four variations of the world model: one with asymmetric reconstruction, and another with symmetric reconstruction. We can conclude that the asymmetric reconstruction further encourages the independence between the state belief and noise belief, The third version lacks the reward objective, in which the performance collapses without the guide of the reward preservation. The fourth operates without the KL divergence objective.

\begin{figure*}[h]
    \centering
    \includegraphics[width=1\linewidth]{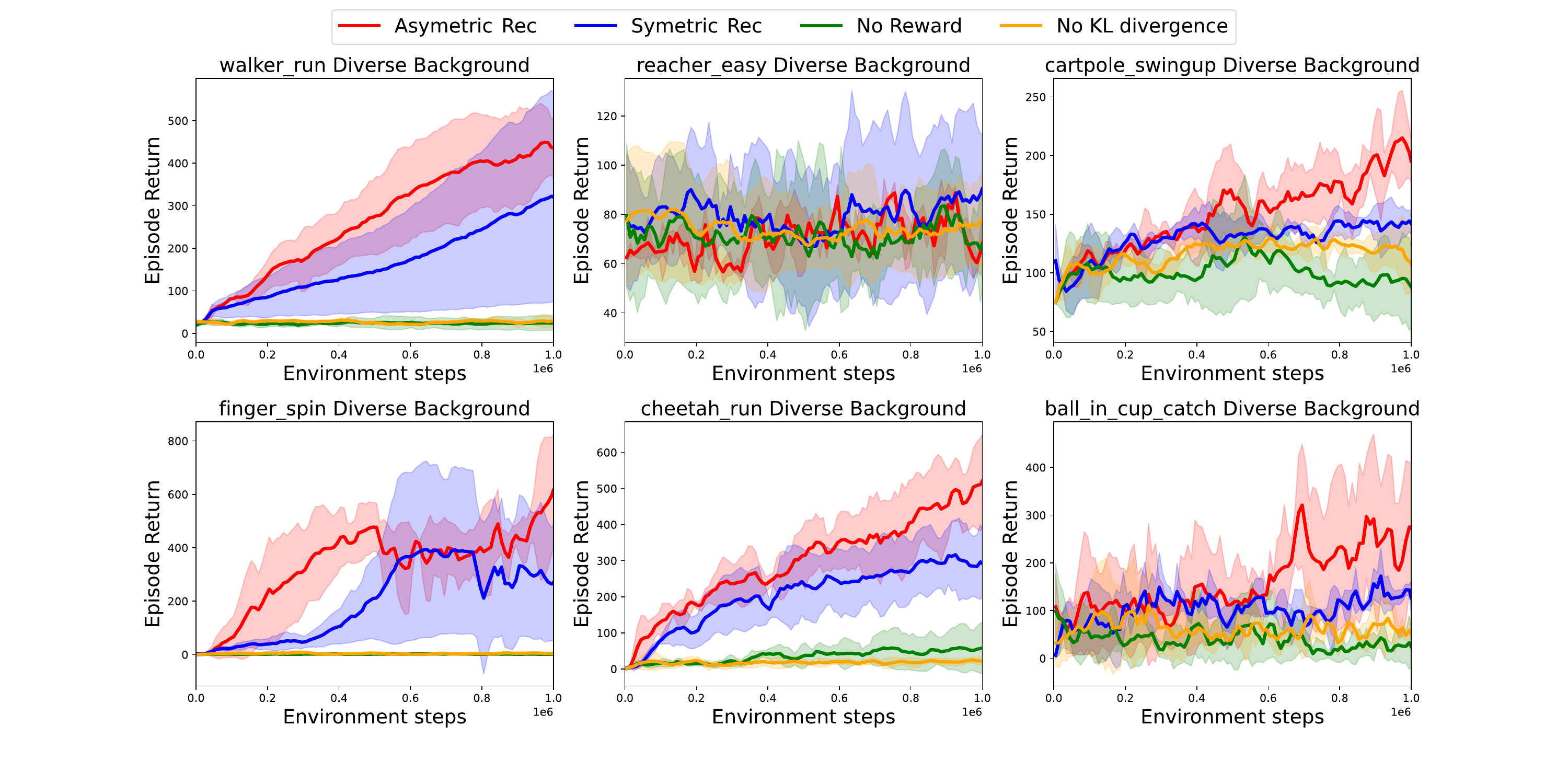}
    \caption{
In a scenario featuring diverse video backgrounds, the Asymmetric world model exhibits superior performance compared to its Symmetric counterpart in the majority of tasks. This superiority is further underscored by the high standard deviation observed in the Symmetric world model's performance, indicating the contribution of asymmetric reconstruction to the transition preservation. Furthermore, models that lack either a reward objective or a KL divergence objective demonstrate the violation of reward preservation and transition preservation, reaffirming the necessity of these components for optimal performance.}
\label{fig:ablation}
\end{figure*}
\section{Experiments Settings and Hyperparameters}
\label{sup:hyper}
All the methods were trained for 1 million environment steps and evaluated every 5,000 environment steps based on the average return from 10 episodes consisting of 1,000 steps each. The evaluation metrics' mean and standard deviation are derived from three independent runs. Each DMC task takes 8 hours on a GTX 3090 GPU, while the RoboDesk task takes 15 hours. We follow the same hyper-parameters ($\alpha, \beta, $ and the dimension of latent state beliefs) from DenoiedMDP without further tuning. This work is built on the top of Denoisedmdp. We don't further tune the hyperparameters. For both state belief and noise dimensions, the deterministic part is 120 and the stochastic part is 20. The $\alpha=1$ and $\beta=1$ are in the Noiseless DMC environments. In the DMC with uniform background, the $\alpha=1$ and $\beta=0.125$ for all tasks except for $\beta=0.25$ in the task `Ball in cup Catch`.  In DMC with diverse background, the $\alpha=1$ and $\beta=0.25$ except for $\beta=0.125$ in the task, `Cheetah run`. In the Robodesk tasks, the $\alpha=2$ and $\beta=0.125$. The code and instructions for the experiments' reproduction are also provided in the supplementary
\end{document}